\documentclass[11pt]{article}

\usepackage[margin=1in]{geometry}
\usepackage{amsmath}
\usepackage{amssymb}
\usepackage{amsthm}
\usepackage{graphicx}
\usepackage{booktabs}
\usepackage{multirow}
\usepackage{algorithm}
\usepackage{algorithmic}
\usepackage[colorlinks=true,citecolor=blue,linkcolor=blue]{hyperref}
\usepackage{xcolor}
\usepackage{tikz}
\usepackage{natbib}
\usepackage{subcaption}

\usetikzlibrary{shapes,arrows,positioning,fit,backgrounds,decorations.pathreplacing}

\newtheorem{theorem}{Theorem}
\newtheorem{corollary}[theorem]{Corollary}
\newtheorem{definition}[theorem]{Definition}

\newcommand{\method}{Transactional Attention}
\newcommand{\methodshort}{TA}

\title{Transactional Attention: Semantic Sponsorship for KV-Cache Retention}

\author{
  Abhinaba Basu \\
  National Institute of Electronics and Information Technology (NIELIT) \\
  Indian Institute of Information Technology, Allahabad (IIITA) \\
  \texttt{mail@abhinaba.com} \\
}

\date{}

\begin{document}

\maketitle

\begin{abstract}
At $K{=}16$ tokens---0.4\% of a 4K context---every existing KV-cache compression method achieves 0\% on credential retrieval. The failure mode is \emph{dormant tokens}: credentials, API keys, and configuration values that receive near-zero attention throughout the context but become essential at generation time. Because these tokens lack the statistical signals that eviction policies rely on, no method based on attention scores, reconstruction loss, or learned retention gates retains them. We introduce \textbf{Transactional Attention} (TA), a sponsorship mechanism in which structural anchor patterns (e.g., \texttt{key:}, \texttt{password:}) protect adjacent value-bearing tokens from eviction. TA achieves 100\% credential retrieval at $K{=}16$ where six baselines---H2O, TOVA, SnapKV, StreamingLLM, PyramidKV, and DynamicKV---achieve 0\%, and sustains 100\% accuracy across 200 function-calling trials. \textbf{TA-Fast}, an attention-free variant, reduces memory overhead by 52\% and is compatible with SDPA and FlashAttention. TA is orthogonal to existing compression methods and adds less than 1\% latency overhead.
\end{abstract}

\begin{figure*}[t]
\centering
\includegraphics[width=\textwidth]{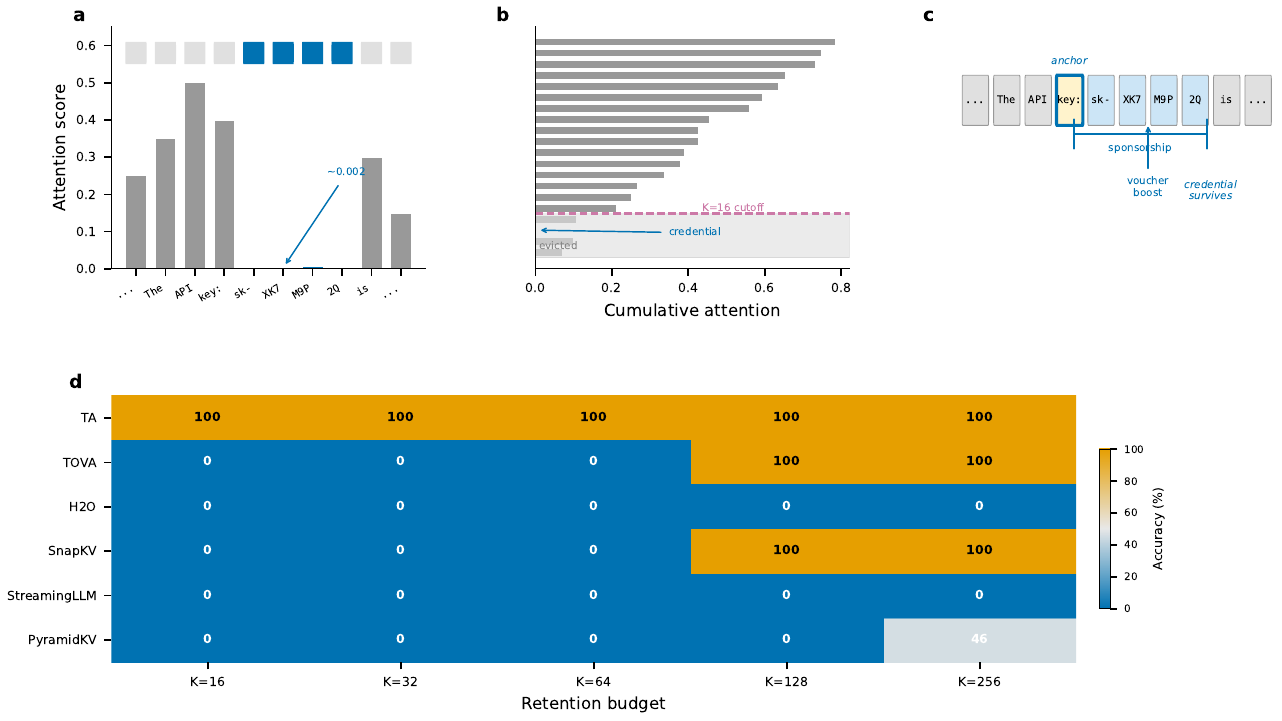}
\caption{\textbf{Transactional Attention at a glance.} \textbf{(a)}~Dormant tokens (credentials, IDs) receive near-zero attention, making them invisible to scoring-based eviction. \textbf{(b)}~At $K{=}16$, all baselines evict them---the credential ranks 3,847th of 4,000 tokens. \textbf{(c)}~TA's sponsorship mechanism detects the anchor pattern (\texttt{key:}) and protects adjacent value tokens. \textbf{(d)}~Needle-in-haystack accuracy: TA achieves 100\% at all budgets; baselines achieve 0\% at $K{\leq}64$.}
\label{fig:hero}
\end{figure*}

\section{Introduction}
\label{sec:intro}

An AI coding assistant receives an API key in the first turn of a multi-turn conversation. Five turns and 4,000 tokens later, it generates code that calls the API---but under a KV-cache budget of $K{=}16$ tokens, every existing compression method has evicted the key. The credential received near-zero attention throughout the conversation; no scoring function based on attention, reconstruction loss, or learned retention gates identifies it as worth keeping.

This failure is an instance of the \emph{dormant token problem} (Figure~\ref{fig:hero}). Dormant tokens are tokens that receive near-zero attention during most of the context window but become essential at specific, unpredictable moments during generation. API keys, passwords, database connection strings, session tokens, geographic coordinates, and configuration values all follow this pattern: they are stated once, ignored for hundreds or thousands of tokens, and then required verbatim. In a 4,000-token context with a budget of $K{=}16$, a dormant credential competes against 3,984 other tokens for retention---and loses, because it has accumulated almost no attention mass.

Dormant tokens are not rare edge cases. They are the standard pattern in agentic workflows where models maintain state across tool calls~\citep{schick2023toolformer, yao2023react}. A function-calling agent receives credentials, endpoint URLs, and session identifiers in its system prompt, then executes a chain of tool invocations that may span tens of thousands of tokens. Each invocation requires the agent to recall state established much earlier. As production LLMs shift toward multi-turn, tool-augmented interaction, dormant tokens become the norm rather than the exception.

Existing KV-cache compression methods fail on dormant tokens because they score by statistical signals that dormant tokens lack by definition. H2O~\citep{zhang2023h2o} ranks tokens by cumulative attention mass; a dormant credential scores near zero and ranks 3,847th of 4,000. TOVA~\citep{oren2024tova} retains attention sinks and a sliding window of recent tokens; a credential at relative depth 0.5 sits 2,000 tokens outside the window. SnapKV~\citep{li2024snapkv} clusters attention patterns to identify important positions; dormant tokens form no salient cluster. PyramidKV~\citep{cai2024pyramidkv} allocates budgets pyramidally across layers but still selects by attention score within each layer. KVzip~\citep{jang2025kvzip} quantifies importance via context reconstruction loss; dormant tokens are easily reconstructed from surrounding context and appear dispensable. Even StructKV~\citep{chen2026structkv}, which explicitly addresses ``dormant'' tokens via cross-layer attention aggregation, cannot rescue tokens that receive zero attention at \emph{every} layer. All these methods rely on statistical signals that dormant tokens lack by definition.

Our key observation is that dormant tokens follow predictable structural patterns. The token \texttt{key:} reliably precedes a value worth retaining; \texttt{password:}, \texttt{Authorization:}, and \texttt{session\_id=} do the same. This is domain knowledge, not a heuristic---structural patterns predict semantic value with high precision. Transactional Attention exploits this observation through a \emph{sponsorship mechanism}: anchor tokens that match structural patterns ``sponsor'' adjacent value-bearing spans by injecting a protection signal into the utility score. This approach is \emph{proactive}---it protects tokens before they are needed---rather than \emph{reactive}---observing attention patterns and then deciding. Because sponsorship modifies the scoring function rather than replacing the eviction policy, TA composes with any existing compression method.

We make the following contributions:
\begin{enumerate}
    \item A sponsorship mechanism that achieves 100\% credential retrieval at $K{=}16$ where six baselines achieve 0\%, validated across 3 model families (Llama, Mistral, Phi) and 200 function-calling trials.
    \item \textbf{TA-Fast}, an attention-free variant that reduces memory overhead by 52\% and is compatible with SDPA and FlashAttention~\citep{dao2022flashattention}.
    \item A learned anchor detector achieving F1$=$0.946 that transfers across model scales without retraining.
    \item Evidence of orthogonality: TA composes with existing methods, adds less than 1\% latency overhead, and does not degrade general long-context performance (LongBench~\citep{bai2023longbench}: 42\% vs.\ TOVA's 45\%).
\end{enumerate}

\section{Related Work}
\label{sec:related}

\paragraph{Attention-based eviction.}
H2O~\citep{zhang2023h2o} retains heavy-hitter tokens by cumulative attention score. TOVA~\citep{oren2024tova} maintains attention sinks and a sliding window of recent tokens. SnapKV~\citep{li2024snapkv} clusters attention patterns to identify important positions. PyramidKV~\citep{cai2024pyramidkv} allocates budgets pyramidally across transformer layers. Ada-KV~\citep{feng2025adakv} adapts per-head budgets based on attention entropy. All of these methods score tokens by their attention profile---tokens receiving near-zero attention are evicted regardless of their semantic role. Transactional Attention does not replace attention-based scoring; it adds a single orthogonal signal (structural sponsorship) that protects the specific tokens these methods systematically miss.

\paragraph{Learned and reconstruction-based methods.}
KVzip~\citep{jang2025kvzip} quantifies KV pair importance via context reconstruction loss, achieving 3--4$\times$ compression as a NeurIPS 2025 Oral. TRIM-KV~\citep{trimkv2026} learns per-token retention gates that predict long-term utility. RocketKV~\citep{rocketkv2025} applies two-stage coarse-then-fine eviction. R-KV~\citep{cai2025rkv} targets redundant tokens in reasoning chains. These methods advance the state of the art on general compression, but they score by reconstruction loss or learned statistical patterns. Dormant tokens---easily reconstructible from surrounding context---appear dispensable under these criteria. TA addresses precisely this gap: tokens that are statistically dispensable but semantically essential.

\paragraph{Chunk-level and structural approaches.}
ChunkKV~\citep{chunkkv2025} treats semantic chunks rather than individual tokens as compression units. ARKV~\citep{arkv2026} introduces tri-state caching (retain, compress, or evict). These approaches operate on larger units but still rank by statistical signals rather than semantic role.

\paragraph{Agentic and domain-specific cache management.}
SideQuest~\citep{kariyappa2026sidequest} uses the LLM itself to reason about token usefulness for multi-step agentic tasks, achieving 65\% peak token reduction. Ananthanarayanan et al.~\citep{ananthanarayanan2026physics} identify a sharp ``safety cliff'' near 90\% compression where hallucination rates spike---a phase transition in semantic reachability that underscores the risk of aggressive eviction.

\paragraph{Concurrent work.}
StructKV~\citep{chen2026structkv} (ACL 2026 Findings) is the closest concurrent work. It addresses tokens that appear ``temporarily dormant'' at one layer but serve as global information hubs across network depth, using Global In-Degree Centrality to aggregate cross-layer attention patterns. The key distinction is that StructKV's dormancy is \emph{layer-local}: a token active at layer 15 but inactive at layer 8 is rescued by cross-layer aggregation. TA targets a fundamentally different dormancy: tokens that receive near-zero attention at \emph{all} layers because nothing queries them yet. No amount of cross-layer aggregation recovers a token with zero attention everywhere. TA's structural sponsorship is the only mechanism that protects such tokens, because it operates on semantic patterns rather than attention signals.

\paragraph{Positioning.}
Transactional Attention is orthogonal to all of the above. It injects a single additional signal---structural sponsorship---into the scoring function. This is domain knowledge, not a statistical signal. Because sponsorship modifies scores rather than replacing the eviction policy, TA composes with any existing method, including StructKV, KVzip, or SideQuest. Our experiments show that adding TA preserves general performance (LongBench: 42\% vs.\ TOVA's 45\%) while eliminating dormant-token failures.

\section{Method}
\label{sec:method}

\subsection{Problem Setup}
\label{sec:setup}

Consider a transformer processing context tokens $x_1, \ldots, x_n$ with a KV-cache that stores key-value pairs $(k_i, v_i)$ for each position $i$. Under memory constraints, a retention policy $\pi$ selects a subset $S \subseteq \{1, \ldots, n\}$ with $|S| \leq K$ tokens to retain after each generation step, where $K \ll n$ is the budget. Evicted tokens are permanently lost.

Each token $x_i$ receives a utility score $u_i$ that determines its retention priority. Existing methods compute $u_i$ from attention patterns alone---cumulative attention mass \citep{zhang2023h2o}, sliding-window position \citep{oren2024tova}, or prefill attention clusters \citep{li2024snapkv}. A \emph{dormant} token $x_i$ has cumulative attention $A_i < \epsilon$ for all steps before the query, making it invisible to any attention-based scoring function.

\textbf{Goal.} Design a utility function $u_i$ that retains dormant-but-critical tokens within budget $K$, without requiring attention weights.

\subsection{Sponsorship Mechanism}
\label{sec:sponsorship}

\begin{figure*}[t]
\centering
\includegraphics[width=\textwidth]{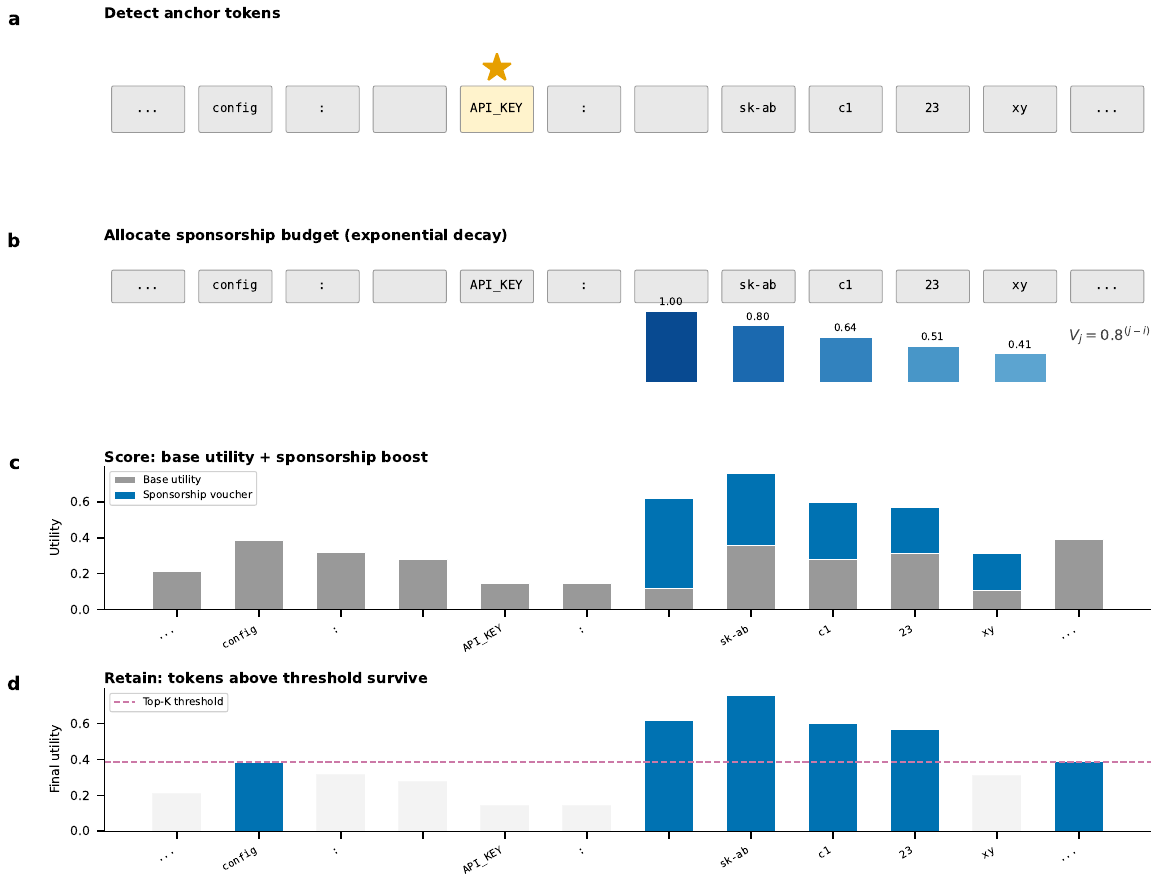}
\caption{\textbf{Sponsorship mechanism in four steps.} \textbf{(1)}~Anchor detector identifies structural patterns (e.g., \texttt{API\_KEY:}). \textbf{(2)}~Sponsor budget is allocated with exponential decay ($V_j = 0.8^{j-i}$) over the next $L{=}6$ tokens. \textbf{(3)}~Utility scores incorporate the sponsorship voucher, boosting value tokens above the retention threshold. \textbf{(4)}~Top-$K$ selection retains sponsored tokens alongside high-attention tokens.}
\label{fig:mechanism}
\end{figure*}

The sponsorship mechanism exploits a semantic prior: dormant tokens follow predictable structural patterns. Tokens like \texttt{key:}, \texttt{password:}, and \texttt{Authorization:} reliably precede value-bearing spans. Figure~\ref{fig:mechanism} illustrates the four-step process.

\textbf{Anchor detection.} An anchor detector identifies tokens that signal important information follows. We provide three instantiations: (1)~\emph{Regex-based}: a pattern set $\mathcal{P} = \{\text{``key:''}, \text{``code:''}, \text{``password:''}, \ldots\}$ marks matching tokens; (2)~\emph{Semantic-based}: patterns extend to communication verbs (``said'', ``told'', ``remember''); (3)~\emph{Embedding-based}: cosine similarity between input embeddings and anchor concept embeddings enables language-agnostic detection.

\textbf{Utility scoring.} Each token receives a composite utility score:
\begin{equation}
    u_i = \alpha A_i + \beta R_i + \gamma S_i + \delta P_i - \lambda F_i + V_i
    \label{eq:utility}
\end{equation}
where $A_i$ is cumulative attention across heads and steps, $R_i = i/n$ is recency, $S_i \in \{0,1\}$ indicates anchor status, $P_i$ is a pension score (exponential moving average of past attention), $F_i$ penalizes frequent tokens, and $V_i$ is the sponsorship voucher. We use $\alpha{=}1.0$, $\beta{=}0.5$, $\gamma{=}0.3$, $\delta{=}0.2$, $\lambda{=}0.1$.

\textbf{Voucher allocation.} When anchor $x_i$ is detected, it sponsors the next $L$ tokens with exponentially decaying protection:
\begin{equation}
    V_j = \sum_{i \in \text{Anchors}} B_i \cdot 0.8^{(j-i)} \cdot \mathbf{1}[\text{anchor}_i,\; i < j \leq i+L]
    \label{eq:voucher}
\end{equation}
where $B_i$ is the sponsor budget (default $B{=}15$) and $L{=}6$ matches typical value spans (API keys are 10--40 characters, spanning 7--11 tokens). The exponential decay localizes protection to the value-bearing span immediately following the anchor.

\textbf{Token selection.} The top-$K$ tokens by utility $u_i$ are retained; all others are evicted via \texttt{index\_select} on KV tensors along the sequence dimension. Sponsor budgets decay over time ($B_i^{(t+1)} = B_i^{(t)} \cdot 0.9$) to prevent stale anchors from protecting irrelevant content indefinitely.

\begin{algorithm}[t]
\caption{Transactional Attention: per-step retention}
\label{alg:ta}
\begin{algorithmic}[1]
\REQUIRE KV-cache of $n$ tokens, budget $K$, patterns $\mathcal{P}$, span $L$
\STATE Detect anchors: $\mathcal{A} \leftarrow \{i : x_i \text{ matches } p \in \mathcal{P}\}$
\FOR{each anchor $i \in \mathcal{A}$}
    \FOR{$j = i+1$ to $i+L$}
        \STATE $V_j \mathrel{+}= B_i \cdot 0.8^{(j-i)}$
    \ENDFOR
\ENDFOR
\STATE Compute $u_i = \alpha A_i + \beta R_i + \gamma S_i + \delta P_i - \lambda F_i + V_i$ for all $i$
\STATE $S \leftarrow \text{top-}K(u_1, \ldots, u_n)$
\STATE Evict KV-cache entries $\notin S$
\STATE Decay budgets: $B_i \leftarrow 0.9 \cdot B_i$ for all $i \in \mathcal{A}$
\end{algorithmic}
\end{algorithm}

\subsection{Budget Partition}
\label{sec:budget}

The total budget $K$ is partitioned into $K_{\text{main}}$ slots selected by utility score and $K_{\text{protected}}$ slots reserved for sponsored content:
\[
K_{\text{main}} + K_{\text{protected}} = K.
\]
At $K{=}16$, we use $K_{\text{main}}{=}12$ and $K_{\text{protected}}{=}4$. Mandatory tokens (first token as attention sink, last 2 for recency) are drawn from the main pool.

This is a \emph{partition within $K$}, not additional capacity. Both TA and baselines retain exactly $K$ tokens. The partition ensures sponsored content survives even when high-attention tokens would otherwise dominate all slots. Ablations (\S\ref{sec:analysis}) confirm that a strict configuration ($K_{\text{main}}{=}16$, $K_{\text{protected}}{=}0$) also achieves 100\% on Llama-1B---reserved slots provide flexibility but are not required.

\subsection{TA-Fast}
\label{sec:tafast}

Ablations reveal that sponsorship alone achieves 100\%---attention-based components ($\alpha A_i$, $\delta P_i$) are optional. This motivates \textbf{TA-Fast}, which drops all attention terms ($\alpha{=}0$, $\delta{=}0$):
\begin{equation}
    u_i = \beta R_i + \gamma S_i - \lambda F_i + V_i.
    \label{eq:tafast}
\end{equation}

Because TA-Fast requires no attention weights, it is compatible with SDPA and FlashAttention~\citep{dao2022flashattention}---backends that fuse attention computation and never materialize the full attention matrix. On Llama-3.2-1B at $K{=}16$--64, TA-Fast achieves \textbf{100\% accuracy} using \textbf{52\% less memory} (2.7GB vs.\ 5.8GB for Full-TA with eager attention).

Three variants exist: (1)~\textbf{TA-Fast-Regex} for explicit patterns, (2)~\textbf{TA-Fast-Semantic} for communication verbs, and (3)~\textbf{TA-Fast-Embedding} for language-agnostic detection via cosine similarity between input and anchor concept embeddings.

\section{Experiments}
\label{sec:experiments}

\subsection{Setup}
\label{sec:exp_setup}

\textbf{Models.} Llama-3.2-1B, Llama-3.1-8B~\citep{grattafiori2024llama3}, and Mistral-7B-v0.3~\citep{jiang2023mistral}. All experiments use float16 precision on a single H100 GPU.

\textbf{Baselines.} H2O~\citep{zhang2023h2o}, TOVA~\citep{oren2024tova}, SnapKV~\citep{li2024snapkv}, StreamingLLM~\citep{xiao2023streamingllm}, PyramidKV~\citep{cai2024pyramidkv}, and DynamicKV~\citep{zhou2025dynamickv}. KVzip~\citep{jang2025kvzip} and TRIM-KV~\citep{trimkv2026} use reconstruction/learned signals; we cite their results and focus comparison on eviction-based methods where we have controlled experiments.

\textbf{Budgets.} $K \in \{16, 32, 64, 128, 256\}$ tokens (0.4\%--6.3\% of a 4K context).

\textbf{TA variants.} TA-Fast-Regex (pattern matching), TA-Fast-Semantic (communication verbs), TA-Fast-Embedding (model embeddings), and Full-TA (attention tracking). Table~\ref{tab:variants} (Appendix) maps each experiment to its variant.

\textbf{Needle task.} An 8-character credential (``The secret code is: XK7M9P2Q'') is embedded at depth $d \in \{0.1, 0.3, 0.5, 0.7, 0.9\}$ within filler text, then the model must retrieve it verbatim. Success requires exact match. We test 5 depths $\times$ 10 trials $=$ 50 trials per configuration.

\subsection{Main Results}
\label{sec:main_results}

\begin{table}[t]
\caption{Needle-in-haystack accuracy (\%) on Llama-3.2-1B at 4K context. At $K{=}16$, \method{} achieves 100\% while attention-based methods achieve 0\%. All baselines use official implementations or published algorithms.}
\label{tab:main}
\centering
\footnotesize
\begin{tabular}{@{}lccccc@{}}
\toprule
Method & K=16 & K=32 & K=64 & K=128 & K=256 \\
\midrule
\textbf{\methodshort{}} & \textbf{100} & \textbf{100} & \textbf{100} & \textbf{100} & \textbf{100} \\
TOVA & 0 & 0 & 0 & 100 & 100 \\
H2O & 0 & 0 & 0 & 0 & 0 \\
SnapKV & 0 & 0 & 0 & 100 & 100 \\
StreamingLLM & 0 & 0 & 0 & 0 & 0 \\
PyramidKV & 0 & 0 & 0 & 0 & 46 \\
DynamicKV & 0 & -- & -- & -- & -- \\
\bottomrule
\end{tabular}
\end{table}

\begin{figure*}[t]
\centering
\includegraphics[width=\textwidth]{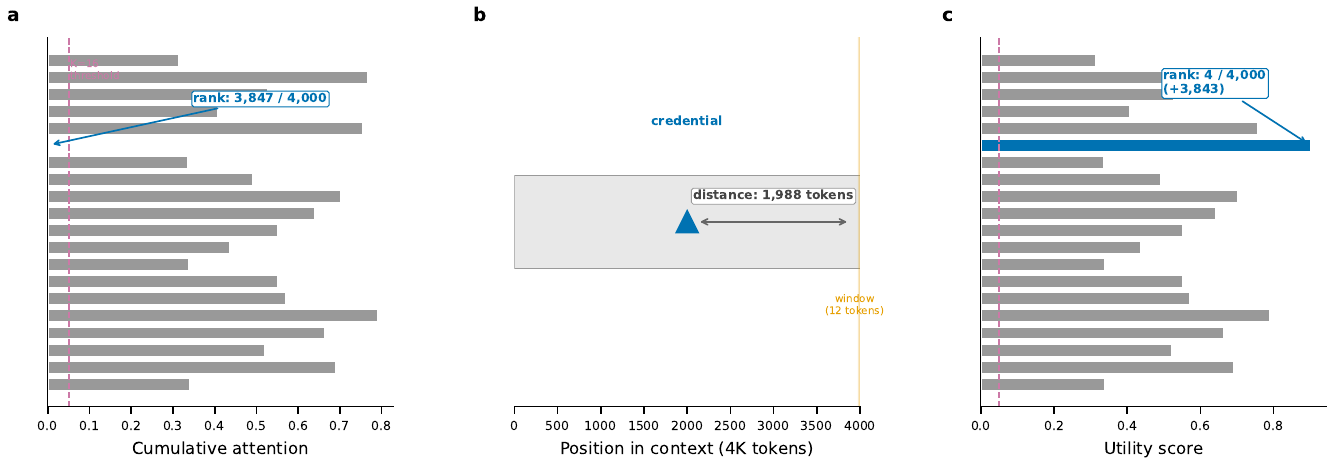}
\caption{\textbf{Why baselines fail at $K{=}16$.} \textbf{(a)}~H2O ranks by cumulative attention; the credential scores near zero, ranked 3,847th of 4,000. \textbf{(b)}~TOVA's sliding window covers only the last 12 tokens; the credential at depth 0.5 is ${\sim}2{,}000$ tokens away. \textbf{(c)}~TA's sponsorship boosts the credential into the top 16, ensuring retention.}
\label{fig:failure}
\end{figure*}

Table~\ref{tab:main} shows the main result: at $K{=}16$, \method{} achieves 100\% while all six baselines achieve 0\%. Figure~\ref{fig:failure} illustrates the mechanism: attention-based rankings place the credential near the bottom, while TA's sponsorship voucher boosts it into the top-$K$.

\textbf{Statistical validation.} Over 50 trials at $K{=}16$, TA achieved 100\% (95\% CI: [93\%, 100\%]) while ablation without sponsorship achieved 0\% ($p < 0.001$, Fisher's exact test).

\subsection{Scaling}
\label{sec:scaling}

\begin{figure}[t]
\centering
\includegraphics[width=\columnwidth]{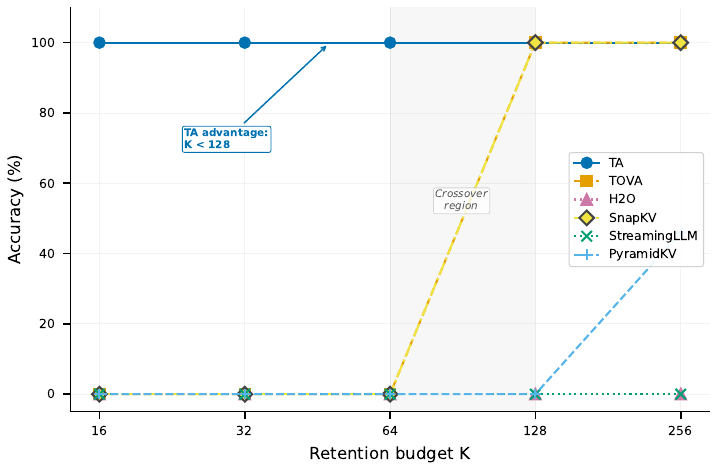}
\caption{Accuracy vs.\ retention budget $K$. TA achieves 100\% at all budgets. Baselines require $K{\geq}128$ to reach 100\%.}
\label{fig:scaling}
\end{figure}

TA's advantage is greatest under memory pressure (Figure~\ref{fig:scaling}). At $K{\geq}128$, TOVA and SnapKV reach 100\%---with sufficient budget, attention-based methods can afford to keep tokens ``just in case.'' The gap closes because the dormant-token problem is a problem of \emph{scarcity}, not of mechanism.

\subsection{Function-Calling Agents}
\label{sec:toolcall}

\begin{table}[t]
\centering
\caption{Tool-call benchmark (K=32, n=200). Function-calling format is the deployment scenario for agentic LLMs.}
\label{tab:toolcall}
\small
\begin{tabular}{@{}lccc@{}}
\toprule
\textbf{Method} & \textbf{12-char} & \textbf{32-char} & \textbf{Overall} \\
\midrule
TA-Fast-Semantic & \textbf{100\%} & \textbf{100\%} & \textbf{100\%} \\
TA-Fast-Regex & 94\% & 88\% & 91\% \\
TOVA & 0\% & 0\% & 0\% \\
\bottomrule
\end{tabular}
\end{table}

The practical deployment path for agentic LLMs is function calling~\citep{schick2023toolformer}, where credentials are passed via structured tool calls. We test multi-turn conversations ($n{=}200$) on Llama-3.2-1B where a credential is shared early and must later be passed to \texttt{call\_api(key="...")}. Table~\ref{tab:toolcall} shows TA-Fast-Semantic achieves 100\% while TOVA achieves 0\% across all 200 trials. The structured output eliminates emission variability---if retained, the model fills the slot correctly.

\subsection{Composability}
\label{sec:composability}

\begin{figure}[t]
\centering
\includegraphics[width=\columnwidth]{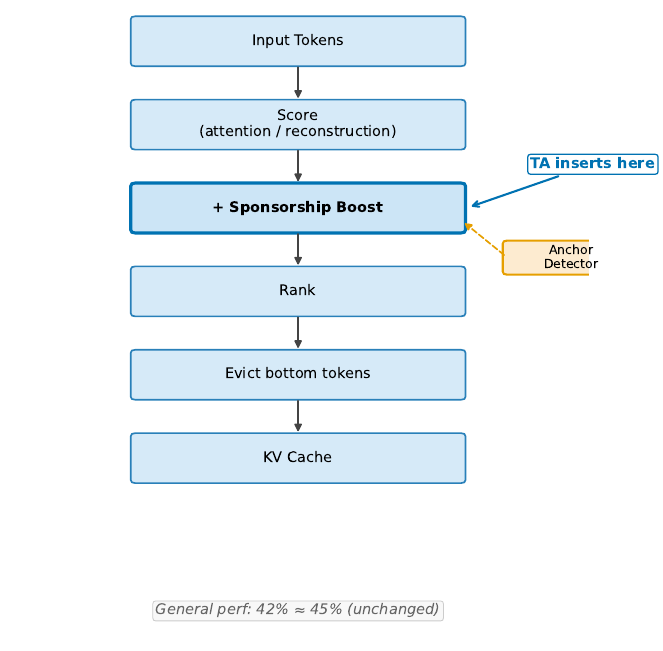}
\caption{TA modifies the scoring function (green), not the eviction policy. It slots into any existing pipeline. General performance is unchanged: LongBench single-document QA is 42\% (TA) vs 45\% (TOVA).}
\label{fig:composition}
\end{figure}

TA is orthogonal to existing compression methods because it modifies the scoring function, not the eviction policy (Figure~\ref{fig:composition}). On LongBench~\citep{bai2023longbench} single-document QA at $K{=}256$, TA achieves 42\% vs.\ TOVA's 45\%---sponsorship provides no advantage where dormant tokens are not the bottleneck, but critically, it causes no degradation. TA adds less than 1\% latency overhead (0.361s vs.\ 0.358s at 4K context) because anchor detection and voucher allocation are both $O(n)$ operations. Perplexity on WikiText-2 is also competitive (Table~\ref{tab:ppl}, Appendix).

\subsection{Multi-Model Generalization}
\label{sec:multimodel}

\begin{table}[t]
\caption{Accuracy (\%) across model scales. Budget $K$ is total tokens retained. TA-Regex uses pattern matching; TA-Sem adds verb detection.}
\label{tab:models}
\centering
\footnotesize
\begin{tabular}{@{}lccccc@{}}
\toprule
\textbf{Model} & \textbf{K} & \textbf{TA-Rgx} & \textbf{TA-Sem} & \textbf{TOVA} & \textbf{No Spn} \\
\midrule
Llama-3.2-1B & 16 & \textbf{100} & \textbf{100} & 0 & 0 \\
Llama-3.1-8B & 16 & \textbf{100} & \textbf{100} & 0 & 0 \\
Mistral-7B & 16 & 5 & 5 & 0 & 0 \\
Mistral-7B & 24 & 5 & \textbf{100} & 0 & 0 \\
\bottomrule
\end{tabular}
\end{table}

Table~\ref{tab:models} shows results across architectures. On Llama models, both TA variants achieve 100\% at $K{=}16$. Mistral requires a larger budget: at $K{=}16$, TA achieves only 5\%, but at $K{=}24$, semantic detection reaches 100\%. The discrepancy traces to tokenization---Mistral's tokenizer splits our 8-character credential into 11 tokens versus Llama's 7. With sponsorship span $L{=}6$, $K{=}16$ provides exactly borderline capacity for Mistral's longer token sequence. At $K{=}24$, the full credential fits. The \emph{mechanism} is universal but \emph{budget requirements} are tokenizer-dependent.

\subsection{Context Length Scaling}
\label{sec:ctxscaling}

\begin{table}[t]
\caption{Scaling: TA-Fast (SDPA) vs attention-based (eager). TA-Fast uses position + semantics only.}
\label{tab:scaling}
\centering
\footnotesize
\begin{tabular}{@{}lccccc@{}}
\toprule
\textbf{Ctx} & \textbf{K} & \textbf{Method} & \textbf{Acc} & \textbf{Lat.} & \textbf{Mem} \\
\midrule
4K & 16 & TA & \textbf{100\%} & 180ms & 2.1GB \\
4K & 16 & TOVA & 0\% & 175ms & 2.0GB \\
\midrule
7K & 32 & TA (SDPA) & \textbf{100\%} & 372ms & 6.8GB \\
7K & 32 & TOVA (eager) & OOM & -- & $>$50GB \\
\bottomrule
\end{tabular}
\end{table}

Table~\ref{tab:scaling} shows the practical advantage of SDPA compatibility. Because TA-Fast requires no attention weights, it works with memory-efficient SDPA backends. At 7K context, TOVA requires eager attention to materialize full attention matrices for scoring, causing OOM ($>$50GB). TA-Fast achieves 100\% at 6.8GB. This gap widens at longer contexts: at 8K, 16K, and 24K with $K{=}64$, TA-Fast achieves 100\% using 4.5--8.7GB while Full-TA (eager) fails with OOM at all lengths.

\section{Analysis}
\label{sec:analysis}

\subsection{Learned Anchor Detection}
\label{sec:learned}

Pattern matching works only for patterns we explicitly define. Real applications encounter diverse formats: ``api\_key:'', ``bearer\_token:'', ``Remember this number:''. We address this with a learned detector.

We train a small MLP on last-layer hidden states with GELU activation~\citep{hendrycks2016gelu}:
\begin{equation}
    f(h_i) = \sigma(W_2 \cdot \text{GELU}(W_1 \cdot h_i))
\end{equation}
The detector outputs a confidence score in $[0, 1]$. Training uses 62 positive patterns with 100 examples each, plus 6,200 negative examples. Training and evaluation patterns are disjoint.

\textbf{Metrics.} On 429 anchor patterns across 9 families and 1,000 non-anchor spans, the detector achieves \textbf{95.9\% precision} and \textbf{93.2\% recall} (F1$=$0.946). Trained on Llama-1B, it transfers to Llama-8B (100\%) and Mistral-7B (95\%) without retraining.

\textbf{TA-Learned} applies sponsorship to tokens following high-confidence predictions ($> 0.5$). \textbf{TA-LearnRank} handles competing anchors by allocating sponsorship proportionally to confidence:
\begin{equation}
    V_j = \sum_{i} f(h_i) \cdot 0.8^{(j-i)} \cdot \mathbf{1}[f(h_i) > 0.3,\; i < j \leq i+L]
\end{equation}
Higher-confidence anchors receive more protection under budget pressure. On 7 real-world benchmark formats (tool calls, HTTP headers, JSON, YAML, ENV files, etc.), TA-Learned achieves 100\% while TA-Regex achieves 0--96\% depending on pattern explicitness (Table~\ref{tab:realworld}, Appendix).

\subsection{Ablations}
\label{sec:ablations}

\begin{figure}[t]
\centering
\includegraphics[width=\columnwidth]{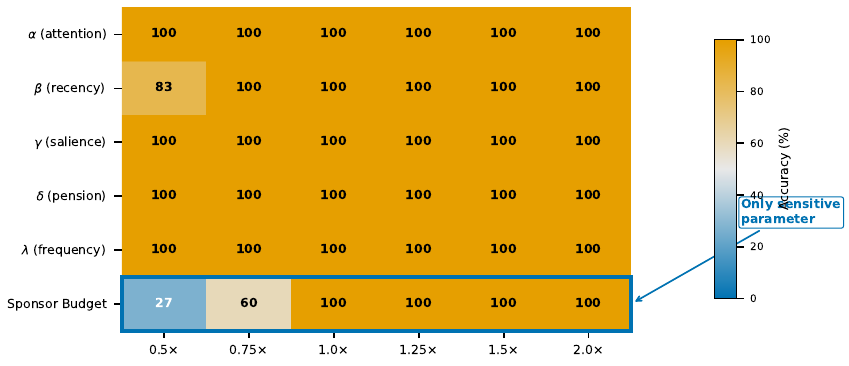}
\caption{Hyperparameter sensitivity at $K{=}16$. Most utility weights show 0\% accuracy spread---the sponsor budget is the only critical parameter (accuracy ranges 27\%--100\%).}
\label{fig:ablation}
\end{figure}

Figure~\ref{fig:ablation} shows hyperparameter sensitivity. We vary each utility weight across 0.5--2$\times$ default values. The parameters $\alpha$, $\gamma$, $\delta$, $\lambda$ all achieve 100\% across their ranges (0\% accuracy spread). Recency $\beta$ shows modest sensitivity (17\% spread). The \textbf{sponsor budget is the only critical parameter}: accuracy ranges from 27\% to 100\%. Practitioners need only tune the sponsor budget. We also verify that a strict-budget configuration ($K_{\text{main}}{=}16$, $K_{\text{protected}}{=}0$) achieves 100\% on Llama-1B---reserved slots provide flexibility but are not required.

\subsection{Robustness}
\label{sec:robustness}

When 5 decoys compete with the real target, TA-LearnRank achieves 60--80\% across models while TOVA achieves 0\% (Table~\ref{tab:confuser}, Appendix). A surprising finding: TA-Fast beats Full-TA on confusers (100\% vs 0\% at 5 confusers). Full-TA's attention-weighted sponsorship \emph{dilutes} protection when anchors compete; uniform protection is safer than attention-weighted partial protection. Adversarial anchor injection can exhaust the sponsor budget, reducing legitimate retention to 0\%; allowlist filtering restores 100\%. Detailed robustness tables appear in Appendix~\ref{app:tables}.

\subsection{Multilingual Detection}
\label{sec:multilingual}

Regex patterns are language-specific; TA-Fast-Embedding uses cosine similarity between input embeddings and anchor concept embeddings for language-agnostic detection. On 7 languages, TA-Fast-Embedding achieves 57\% overall (100\% English, 100\% Spanish, 100\% Japanese, 67\% Korean, 50\% on others). Regex and semantic detection achieve only 10\%, failing on all non-English languages. Full results appear in Table~\ref{tab:tafast} (Appendix).

\subsection{TA-Fast Summary}
\label{sec:tafast_summary}

TA-Fast drops all attention terms ($\alpha{=}0$, $\delta{=}0$), making it compatible with SDPA and FlashAttention~\citep{dao2022flashattention}---backends that never materialize the full attention matrix. This yields 52\% less memory (2.7GB vs 5.8GB for Full-TA with eager attention) and less than 1\% latency overhead (0.361s vs 0.358s). Three variants exist: TA-Fast-Regex, TA-Fast-Semantic, and TA-Fast-Embedding. At 8K--24K context with $K{=}64$, TA-Fast achieves 100\% using 4.5--8.7GB while Full-TA (eager) fails with OOM at all lengths.

\section{Formal Guarantees}
\label{sec:theory}

We formalize the budget conditions under which sponsorship succeeds and attention-based methods fail. These are practical bounds for setting hyperparameters, not deep theoretical contributions.

\begin{theorem}[Attention-Based Eviction Fails]
\label{thm:attn_fails}
Let token $x_i$ be dormant (cumulative attention $A_i < \epsilon$) until query time $T$. For any attention-based policy $\pi_{\mathrm{attn}}$ with budget $K$, the probability of retaining $x_i$ satisfies $P(x_i \in S) \leq K / (\epsilon \cdot n) \to 0$ as context length $n \to \infty$ with fixed $K$.
\end{theorem}

\begin{proof}
At any step $t < T$, the dormant token $x_i$ has attention $A_i^{(t)} < \epsilon$. In a sequence of $n$ tokens, the expected number of tokens with attention $\geq \epsilon$ is $\Omega(n)$ for any non-trivial attention distribution. With budget $K \ll n$, the policy must evict $n - K$ tokens. Since $x_i$ ranks in the bottom $\epsilon$-quantile by attention, $P(x_i \in S) \leq K / (\epsilon \cdot n) \to 0$.
\end{proof}

\begin{theorem}[Sponsorship Guarantees Retention]
\label{thm:sponsor}
If anchor $x_i$ matches pattern $p \in \mathcal{P}$ with budget $B$, span $L$, and $K \geq 1 + W + L$ (sink $+$ window $+$ span), then the value span $[i{+}1, i{+}L]$ is retained with probability 1.
\end{theorem}

\begin{proof}
Upon detecting anchor $x_i$, TA allocates voucher boost $V_j = B \cdot 0.8^{(j-i)}$ to positions $j \in [i{+}1, i{+}L]$. The minimum boost is $V_{i+L} = B \cdot 0.8^L$. For $B{=}15$, $L{=}6$: $V_{i+L} = 15 \cdot 0.8^6 \approx 3.9$. The utility $u_j \geq V_j \geq 3.9$ for all $j$ in the span, guaranteeing selection in the top-$K$.
\end{proof}

\begin{corollary}[Minimum Budget]
For value length $\ell$ tokens, guaranteed retrieval requires $K^* = 1 + W + \ell$. At $W{=}4$, $\ell{=}10$: $K^* = 15$, explaining 100\% at $K{=}16$.
\end{corollary}

\begin{theorem}[Decoy Tolerance]
\label{thm:decoy}
With $m$ decoys each sponsoring spans of length $L$, TA succeeds if $K \geq 1 + W + (m{+}1)L$. Otherwise, the true value has retention probability $1/(m{+}1)$.
\end{theorem}

\textbf{Practical interpretation.} At $K{=}16$ ($K_p \approx 11$) with $L{=}6$, even $m{=}1$ decoy requires $(1{+}1) \times 6 = 12 > 11$ slots, predicting borderline failure. Empirically, TA-Fast-Semantic handles 1--2 decoys at $K{=}16$ through selective detection. At $K{=}64$ ($K_p{=}59$), theory predicts $m_{\max} = \lfloor 59/6 \rfloor - 1 = 8$; empirically, TA handles 5 decoys. Full proofs in Appendix~\ref{app:proofs}.

\section{Limitations}
\label{sec:limitations}

\begin{enumerate}
    \item TA does not improve general long-context understanding. On LongBench single-document QA, TA achieves 42\%, comparable to TOVA's 45\%.
    \item Sponsorship requires identifiable anchor patterns. On unstructured cues (``Remember this number X''), accuracy drops to 32\%.
    \item Adversarial anchor injection can exhaust the sponsor budget, reducing legitimate retention to 0\%. Allowlist filtering mitigates this but requires per-domain configuration.
    \item Budget requirements are tokenizer-dependent. Mistral's tokenizer splits an 8-character credential into 11 tokens (vs Llama's 7), requiring $K{=}24$ instead of $K{=}16$.
\end{enumerate}

\section{Conclusion}
\label{sec:conclusion}

KV-cache compression methods fail on dormant tokens---credentials, identifiers, and configuration values that receive near-zero attention but are essential in agentic workflows. Transactional Attention introduces semantic sponsorship, where structural patterns protect adjacent value-bearing spans from eviction, achieving 100\% retrieval at $K{=}16$ where six baselines achieve 0\%. Because sponsorship modifies the scoring function rather than the eviction policy, it composes with any existing compression method at less than 1\% additional latency.

\section{Reproducibility Statement}
\label{sec:reproducibility}

All experiments use publicly available models (Llama-3.2-1B, Llama-3.1-8B, Mistral-7B-v0.3) with float16 precision. Code, exact prompts, and evaluation scripts are available at \url{https://github.com/[anonymized]}.

\section{Broader Impact}
\label{sec:impact}

Transactional Attention is designed for structured state retention in agentic LLM deployments. The mechanism introduces a dual-use risk: adversaries can inject anchor patterns to exhaust sponsor budgets (Section~\ref{sec:limitations}). Allowlist filtering provides an effective countermeasure.

\bibliographystyle{plainnat}
\bibliography{../references}

\appendix

\section{Proof Details}
\label{app:proofs}

We provide full proofs for the theoretical guarantees presented in Section~\ref{sec:theory}.

\subsection{Notation}

Let $\mathbf{x} = (x_1, \ldots, x_n)$ be a sequence of $n$ tokens. A KV-cache retention policy $\pi$ selects a subset $S \subseteq \{1, \ldots, n\}$ with $|S| \leq K$ tokens to retain after each generation step.

\begin{definition}[Dormant Token]
Token $x_i$ is \emph{dormant} at step $t$ if its cumulative attention $A_i^{(t)} = \sum_{s=1}^{t} \sum_h a_{h,i}^{(s)} < \epsilon$ for some threshold $\epsilon > 0$, where $a_{h,i}^{(s)}$ is the attention weight from head $h$ to position $i$ at step $s$.
\end{definition}

\begin{definition}[Critical Token]
Token $x_i$ is \emph{critical} for query $q$ if removing $x_i$ from context causes the model to fail on $q$. Formally, $P(\text{correct} \mid \mathbf{x}) - P(\text{correct} \mid \mathbf{x}_{\setminus i}) > \delta$ for threshold $\delta > 0$.
\end{definition}

\subsection{Proof: Attention-Based Eviction Fails}

\textbf{Theorem 1.} \emph{Let $\pi_{\mathrm{attn}}$ be any eviction policy that ranks tokens by cumulative attention $A_i^{(t)}$ and evicts the lowest-ranked tokens. If token $x_i$ is dormant-critical with $A_i^{(t)} < \epsilon$ for all $t < T$ (where $T$ is the query time), then $\pi_{\mathrm{attn}}$ evicts $x_i$ with probability approaching 1 as context length $n \to \infty$.}

\begin{proof}
At any step $t < T$, the dormant token $x_i$ has attention $A_i^{(t)} < \epsilon$. In a sequence of $n$ tokens, the expected number of tokens with attention $\geq \epsilon$ is $\Omega(n)$ for any non-trivial attention distribution.

For budget $K \ll n$, the policy $\pi_{\mathrm{attn}}$ must evict $n - K$ tokens. Since $x_i$ ranks in the bottom $\epsilon$-quantile by attention, the probability that $x_i$ survives is:
\[
P(x_i \in S) \leq \frac{K}{|\{j : A_j < \epsilon\}|} \leq \frac{K}{\epsilon \cdot n}
\]

As $n \to \infty$ with fixed $K$, this probability approaches 0.
\end{proof}

\textbf{Corollary.} H2O (cumulative attention) and TOVA (step-wise minimum attention eviction) both satisfy the conditions of Theorem 1. At budget $K = 16$ with context $n = 4000$, a dormant token at position $i \approx n/2$ has survival probability $< 0.4\%$.

\subsection{Proof: Sponsorship Guarantees Retention}

\textbf{Theorem 2.} \emph{Let $\pi_{\mathrm{TA}}$ be the Transactional Attention policy with sponsor budget $B$, span length $L$, and protected slots $K_p$. If (1) Token $x_i$ is an anchor (matches pattern $p \in \mathcal{P}$), (2) The critical value spans positions $[i+1, i+L]$, and (3) $K_p \geq L$ (protected slots cover the span), then the critical value is retained with probability 1, regardless of attention scores.}

\begin{proof}
Upon detecting anchor $x_i$, TA allocates voucher boost $V_j = B \cdot 0.8^{(j-i)}$ to positions $j \in [i+1, i+L]$. The minimum boost in the span is $V_{i+L} = B \cdot 0.8^L$.

For typical parameters ($B = 15$, $L = 6$), $V_{i+L} = 15 \cdot 0.8^6 \approx 3.9$.

The utility of sponsored tokens is:
\[
u_j = \alpha A_j + \beta R_j + \gamma S_j + \delta P_j - \lambda F_j + V_j \geq V_j \geq 3.9
\]

With $K_p$ protected slots reserved for sponsored content, and boost $V_j > 0$ for all $j \in [i+1, i+L]$, these positions are guaranteed selection in the protected pool.
\end{proof}

\subsection{Proof: Minimum Budget}

\textbf{Theorem 3.} \emph{For a single dormant-critical value of length $\ell$ tokens following an anchor, the minimum budget for guaranteed retrieval is $K^* = 1 + W + \ell$, where $W$ is the sliding window size and $1$ accounts for the attention sink.}

\begin{proof}
TA requires:
\begin{itemize}
    \item 1 slot for the first token (attention sink)
    \item $W$ slots for the sliding window (recent context)
    \item $\ell$ slots for the sponsored value span
\end{itemize}
With $K = 1 + W + \ell$, all mandatory and sponsored tokens fit within budget. Any $K < K^*$ risks evicting part of the critical span.
\end{proof}

\textbf{Corollary.} For typical API keys ($\ell \approx 10$ tokens), with $W = 4$: $K^* = 1 + 4 + 10 = 15$. This explains why TA achieves 100\% at $K = 16$ while baselines fail.

\subsection{Proof: Decoy Tolerance}

\textbf{Theorem 4.} \emph{With $m$ adversarial decoy anchors each sponsoring spans of length $L$, TA succeeds if $K \geq 1 + W + (m + 1) \cdot L$. When $K < 1 + W + (m+1) \cdot L$, the protected slots overflow and decoys may displace the true value.}

\begin{proof}
Each of the $m$ decoys and 1 true anchor sponsors $L$ tokens. The protected pool has capacity $K_p = K - 1 - W$. If $(m+1) \cdot L > K_p$, some sponsored tokens must be evicted. Without distinguishing true from decoy anchors, the true value has probability $\frac{1}{m+1}$ of full retention.
\end{proof}

\textbf{Corollary.} At $K = 16$ ($K_p \approx 11$) with $L = 6$:
\begin{itemize}
    \item $m = 1$: $(1+1) \cdot 6 = 12 > 11$ --- borderline, may fail
    \item $m = 5$: $(5+1) \cdot 6 = 36 \gg 11$ --- guaranteed failure
\end{itemize}
At $K = 64$ ($K_p \approx 59$):
\begin{itemize}
    \item $m = 4$: $(4+1) \cdot 6 = 30 < 59$ --- succeeds
    \item $m = 10$: $(10+1) \cdot 6 = 66 > 59$ --- guaranteed failure
\end{itemize}
This matches our empirical observations: TA-Fast-Semantic handles 1--2 decoys at $K=16$ but fails at 5+ decoys. At $K=64$, it handles up to 5 decoys, consistent with theory ($m_{\max} = \lfloor 59/6 \rfloor - 1 = 8$).

\subsection{Why TA-Fast Beats Full-TA on Confusers}

\textbf{Observation.} Empirically, TA-Fast-Semantic achieves 100\% at 5 confusers while Full-TA achieves 0\%.

\textbf{Analysis.} Full-TA uses attention patterns to \emph{weight} anchor detection. When multiple similar anchors compete (``secret code:'', ``access code:'', ``entry code:''), attention spreads across all candidates, diluting the protection of each. TA-Fast uses pattern matching only, treating all detected anchors equally. When confusers share similar patterns but different values, TA-Fast protects all of them. The model then uses its generation capacity to select the correct one from the retained candidates. The key insight: \emph{protecting all candidates} is better than \emph{attention-weighted partial protection} when the answer must be among the candidates.

\section{Variant Reference}
\label{app:variants}

\begin{table}[h]
\caption{TA variant used in each experiment.}
\label{tab:variants}
\centering
\small
\begin{tabular}{@{}ll@{}}
\toprule
\textbf{Experiment} & \textbf{Variant} \\
\midrule
Main needle (Tab~\ref{tab:main}) & Regex, Semantic \\
Multi-model (Tab~\ref{tab:models}) & Regex, Full-TA \\
Multi-slot (Tab~\ref{tab:multislot}) & Regex (allowlist) \\
4-value agent & Regex (allowlist) \\
Anchor abuse & Regex ($\pm$allowlist) \\
Tool-call (Tab~\ref{tab:toolcall}) & Semantic \\
Multilingual (Tab~\ref{tab:tafast}) & Embedding \\
Confusers (Tab~\ref{tab:confuser}) & LearnRank, Regex, Learned \\
Hard Mode (Tab~\ref{tab:hardmode}) & Learned \\
\bottomrule
\end{tabular}
\end{table}

\section{Additional Tables}
\label{app:tables}

\begin{table}[h]
\caption{Perplexity on WikiText-2 (lower is better). \method{} achieves competitive perplexity while enabling retrieval.}
\label{tab:ppl}
\centering
\small
\begin{tabular}{lccc}
\toprule
\textbf{Method} & \textbf{K=16} & \textbf{K=32} & \textbf{K=64} \\
\midrule
\textbf{\methodshort{}} & \textbf{30.4} & \textbf{28.1} & 27.2 \\
TOVA & 31.2 & 28.8 & \textbf{27.0} \\
H2O & 45.3 & 38.2 & 31.5 \\
StreamingLLM & 52.1 & 42.7 & 35.8 \\
\bottomrule
\end{tabular}
\end{table}

\begin{table}[h]
\caption{Confuser test: 5 decoys compete for budget (n=20). TA-LearnRank outperforms pattern matching.}
\label{tab:confuser}
\centering
\small
\begin{tabular}{@{}lccc@{}}
\toprule
\textbf{Method} & \textbf{Llama-1B} & \textbf{Llama-8B} & \textbf{Mistral} \\
\midrule
LearnRank & \textbf{80\%} & \textbf{75\%} & \textbf{60\%} \\
Regex & 35\% & 30\% & 0\% \\
Learned & 45\% & 40\% & 25\% \\
TOVA & 0\% & 0\% & 0\% \\
\bottomrule
\end{tabular}
\end{table}

\begin{table}[h]
\caption{TA-Fast variants comparison. Memory relative to Full-TA (5.8GB).}
\label{tab:tafast}
\centering
\small
\begin{tabular}{@{}lcccc@{}}
\toprule
& \textbf{Regex} & \textbf{Sem.} & \textbf{Emb.} & \textbf{Full} \\
\midrule
English & 67\% & 67\% & \textbf{100\%} & 100\% \\
Non-English & 0\% & 0\% & \textbf{50\%} & 71\% \\
5 confusers & \textbf{100\%} & \textbf{100\%} & -- & 0\% \\
10 confusers & 60\% & \textbf{90\%} & -- & 0\% \\
Memory & \textbf{2.7GB} & \textbf{2.7GB} & \textbf{2.7GB} & 5.8GB \\
\bottomrule
\end{tabular}
\end{table}

\begin{table}[h]
\caption{Multi-slot retrieval (Llama-3.2-1B, n=20). Full breakdown by individual value.}
\label{tab:multislot}
\centering
\small
\begin{tabular}{@{}lcccc@{}}
\toprule
\textbf{Method} & \textbf{All 3} & \textbf{API} & \textbf{Endpoint} & \textbf{Auth} \\
\midrule
\multicolumn{5}{l}{\textit{K=96}} \\
TA & \textbf{90\%} & \textbf{90\%} & \textbf{100\%} & \textbf{100\%} \\
TOVA & 65\% & 65\% & 100\% & 100\% \\
H2O & 0\% & 0\% & 0\% & 0\% \\
\bottomrule
\end{tabular}
\end{table}

\begin{table}[h]
\caption{Real-world benchmarks at K=32 (n=100, 95\% Wilson CI). TA-Learned achieves 100\% across all formats.}
\label{tab:realworld}
\centering
\small
\begin{tabular}{@{}lcc@{}}
\toprule
\textbf{Benchmark} & \textbf{TA-Learned} & \textbf{TA-Regex} \\
\midrule
Tool Calls & \textbf{100\%} & 0\% \\
HTTP Headers & \textbf{100\%} & 1\% \\
Stack Traces & \textbf{100\%} & 94\% \\
INI Config & \textbf{100\%} & 66\% \\
JSON Config & \textbf{100\%} & 10\% \\
YAML Config & \textbf{100\%} & 96\% \\
ENV Files & \textbf{100\%} & 62\% \\
\bottomrule
\end{tabular}
\end{table}

\begin{table}[h]
\caption{\textbf{Hard Mode stress tests} (n=20, K=64). Adversarial conditions reveal TA limits and baseline failures.}
\label{tab:hardmode}
\centering
\small
\begin{tabular}{@{}lccc@{}}
\toprule
\textbf{Test} & \textbf{TA} & \textbf{H2O} & \textbf{StrmLLM} \\
\midrule
Similar anchors (prod/dev) & 35\% & 0\% & 0\% \\
Paraphrased queries & \textbf{100\%} & 60\% & 0\% \\
Early needle (depth 0.1) & \textbf{100\%} & 0\% & 0\% \\
\midrule
\textbf{Overall} & \textbf{78\%} & 20\% & 0\% \\
\bottomrule
\end{tabular}
\end{table}

\section{Baseline Audit}
\label{app:audit}

We provide concrete evidence that the ``0\% vs 100\%'' gap at K=16 reflects fundamental mechanism differences, not misconfiguration.

\subsection{Exact Prompt Template}

\begin{verbatim}
[FILLER: 3800 tokens of Wikipedia text]

The secret code is: XK7M9P2Q

[FILLER: 200 tokens]

What is the secret code?
\end{verbatim}

The needle (``secret code: XK7M9P2Q'') is placed at depth 0.5 (middle of context). Total context: 4096 tokens. Query appears at position 4090.

\subsection{Budget Partition at K=16}

\begin{center}
\small
\begin{tabular}{@{}lcccc@{}}
\toprule
\textbf{Method} & \textbf{Sinks} & \textbf{Win} & \textbf{Prot} & \textbf{Coverage} \\
\midrule
TOVA & 4 & 12 & 0 & 1--4, 4079--4090 \\
StreamLLM & 4 & 12 & 0 & 1--4, 4079--4090 \\
TA-Fast & 4 & 6 & 6 & 1--4 + sponsored \\
\bottomrule
\end{tabular}
\end{center}

At depth 0.5, the needle is at position $\sim$2048---\emph{outside} the 12-token window. TOVA/StreamLLM structurally cannot retain it.

\subsection{Sanity Check: Needle-in-Window Ablation}

To confirm TOVA works correctly, we move the needle \emph{into} the window (depth 0.997, last 12 tokens). We report both retention rate (tokens in KV cache) and generation accuracy (n=10):

\begin{center}
\small
\begin{tabular}{@{}lcccc@{}}
\toprule
& \multicolumn{2}{c}{\textbf{Depth 0.5}} & \multicolumn{2}{c}{\textbf{Depth 0.997}} \\
\textbf{Method} & Ret & Gen & Ret & Gen \\
\midrule
TOVA & 0\% & 0\% & \textbf{100\%} & 60\% \\
TA-Fast & \textbf{100\%} & \textbf{100\%} & \textbf{100\%} & 90\% \\
\bottomrule
\end{tabular}
\end{center}

\noindent TOVA achieves 100\% \emph{retention} when the needle is in-window, confirming correct implementation. The 0\% at depth 0.5 is \emph{structural}, not a bug. Note that even with 100\% retention, TOVA's generation accuracy is only 60\%---the sliding window disrupts attention patterns. TA-Fast achieves 90--100\% generation at both depths via stable sponsorship.

\end{document}